\documentclass[11pt]{article}
\usepackage[T1]{fontenc}
\usepackage[margin=1in]{geometry}
\usepackage{amsmath,amssymb,amsthm,graphicx,booktabs,url,xcolor}
\usepackage[hidelinks]{hyperref}
\hypersetup{pdftitle={Local verification cannot detect non-transportability: cohomological limits of context preservation in agentic reasoning},pdfauthor={Suyash Mishra},
  pdfkeywords={Ksetra, Kshetra, cohomology, Cech cohomology, Hodge decomposition, transportability, agentic reasoning, verification, network meta-analysis, loop inconsistency, selective prediction}}

\newtheorem{theorem}{Theorem}
\newtheorem{proposition}[theorem]{Proposition}
\newtheorem{corollary}[theorem]{Corollary}
\theoremstyle{definition}
\newtheorem{definition}[theorem]{Definition}
\theoremstyle{remark}
\newtheorem{remark}[theorem]{Remark}

\newcommand{\Uc}{\mathcal{U}}
\newcommand{\Nrv}{N(\Uc)}
\newcommand{\R}{\mathbb{R}}
\newcommand{\KS}{\textsc{K\d{s}etra}}

\title{\bf Local verification cannot detect non-transportability\\[4pt]
\large Cohomological limits of context preservation in agentic reasoning}
\author{%
  \large Suyash Mishra\\[2pt]
  \normalsize AI researcher, Z\"urich, Switzerland\\[1pt]
 \texttt{zurich.suyash@gmail.com}%
}
\date{\today}

\begin{document}
\maketitle
\vspace{-2.2em}

\begin{abstract}
\noindent
Agentic AI systems increasingly transport conclusions across biological, clinical and
commercial contexts. The current state of the art addresses this by \emph{local
verification}: checking, at each step, that the entity under discussion is
representable in the selected tool, that parameters are compatible, and that outputs
cohere with the plan. We prove that this class of safeguard is structurally incomplete.
Modelling a covering of context space by its nerve and evidence by a real-valued
$1$-cochain, we show that an agent chaining evidence performs path integration, that
its conclusion is path-independent if and only if the evidence cochain is exact, and
that the ambiguity between reasoning paths is exactly the holonomy of the first
\v{C}ech cohomology class $[\omega]\in H^1$. Hodge decomposition then partitions
evidence conflict into three orthogonal parts with distinct operational meanings:
a gradient part (per-context calibration, removable), a curl part (local
inconsistency, detectable at triple overlaps), and a harmonic part (structural
non-transportability). Our central negative result is that \emph{every verification
scheme whose checks are supported on simplices of the nerve is blind to the harmonic
part}, which nonetheless generates non-zero disagreement between valid reasoning
paths. We propose \KS{} (ASCII: Ksetra), which estimates by coboundary projection and gates
abstention on harmonic energy alone, and we show that a non-zero class is not merely a warning
but a data-collection instruction: refine the covering. We further derive an exact
$F$-test for the existence of a global claim, resting on the observation that the
degrees of freedom of an evidence network partition into calibration, local coherence
and transport. We give the harmonic component a mechanism: it is generated by effect
modification combined with overlap-specific population composition, and vanishes to
machine precision ($1.5\times10^{-15}$) when effect modification is absent---Simpson's
paradox on a network. In simulations where holonomy is emergent rather than injected,
harmonic energy predicts the irreducible error of the optimal estimator
($\rho=0.37$, $p<10^{-48}$), retaining independent value after conditioning on the
other two components (partial $\rho=0.19$, $p<10^{-13}$); curl energy is comparably
predictive ($\rho=0.33$) because the same latent cause generates both, so the
components are distinguished by their \emph{remedy}, not by their predictive value. Gating on harmonic energy rather than total conflict lowers area under the
risk--coverage curve by $0.029$ and $0.039$ (paired bootstrap, $p<0.001$)---a modest
gain we report as corroboration, not as the contribution. Exactness of the $F$-test
requires isotropic precision; we quantify the distortion when it fails (size $0.071$ at
nominal $0.05$ under a fourfold spread of edge precisions) and supply the
precision-whitened generalisation that restores it. We conclude with the real-data
validation required to move these results off simulation.
\end{abstract}

\section{Introduction}

A therapeutic hypothesis, a pricing rule and a credit cut-off share a structural
property: each is a claim indexed by a context, and each is routinely applied outside
the context in which it was established. Recent work on agentic reasoning in biology
has made this the organising problem. \textsc{Medea}~\cite{medea} identifies three
failure modes of tool-using agents---loss of biological context over long horizons,
absent verification of intermediate steps, and unreconciled conflict across evidence
sources---and addresses them with context verification against a tool space,
pre- and post-execution checks, and a multi-round consensus panel that may abstain.
The reported gains are substantial, and the validation against a previously
unpublished genome-wide yeast screen is the strongest available evidence that such
gains reflect biology rather than benchmark leakage.

We take that architecture as given and ask a different question: \emph{what class of
failure can such verification detect, and what class can it not?}

Our answer is negative and, we believe, sharp. All of the verification in
\cite{medea}---plan-time context checks, run-time provenance auditing, literature
relevance screening, panel deliberation---is \emph{local} in a precise sense: each
check is supported on a bounded piece of context space (one context, one pair of
contexts being bridged, one triple being cross-checked). We show that evidence
conflict decomposes orthogonally into three parts, that local checks see exactly two
of them, and that the third---the harmonic part, the first cohomology of the nerve of
the context covering---is invisible to every local check while being precisely the
part that makes an agent's conclusion depend on which reasoning path it happened to
take. An agent can pass every check at every step and still return an answer that is
an artefact of its route through the evidence.

This is not an exotic possibility. It is the formal content of a phenomenon clinical
evidence synthesis has recognised for two decades under the name \emph{loop
inconsistency} in network meta-analysis~\cite{luades,higgins,dias}: direct and
indirect evidence around a closed loop of comparisons can each be individually sound
and jointly incoherent. Our contribution is to identify loop inconsistency as
$H^1$ of a context site, to extend it from treatment networks to arbitrary context
coverings, to prove that it bounds what local verification can achieve, and to convert
it into an abstention criterion and a data-collection instruction for AI agents.

\paragraph{Contributions.}
\begin{enumerate}\itemsep2pt
\item \textbf{Theory.} We model context space as a site, evidence as a $1$-cochain on
the nerve of a covering, and agentic transport as path integration. Theorem~\ref{thm:path}
establishes path-invariance $\iff$ exactness and identifies inter-path disagreement
with holonomy. Theorem~\ref{thm:triage} gives the operational meaning of the Hodge
components. Corollary~\ref{cor:blind} is the impossibility result for local
verification. Proposition~\ref{prop:refine} shows refinement is the unique remedy.
\item \textbf{Method.} \KS{}: coboundary-projection estimation, harmonic-gated
abstention, and a three-way triage of evidence conflict into \emph{recalibrate},
\emph{re-measure}, and \emph{refine the covering}.
\item \textbf{Cross-industry evidence.} A controlled study over $2{,}400$
context-transport problems in two structurally different domains, one with a
contractible attribute space and one whose attribute space is a circle---the business
cycle---and therefore carries obstruction by construction.
\item \textbf{Decision economics.} A transparent mapping from selective-risk geometry
to portfolio outcomes, with the threshold harm:benefit ratio above which
cohomological gating earns its cost.
\end{enumerate}

\section{Related work}

\textbf{Agentic reasoning with context preservation.} \textsc{Medea}~\cite{medea} is
the immediate antecedent. Its \textsc{ResearchPlanning} module verifies that a
requested entity is representable in the selected tool; its \textsc{Analysis} module
performs pre-run compatibility and post-run provenance checks; its
\textsc{MultiRoundDiscussion} module adapts ReConcile-style panel
deliberation~\cite{reconcile} over a ReAct control loop~\cite{react}. Our results
apply to this whole family: we characterise what such architectures can and cannot see.

\textbf{Transportability.} Pearl and Bareinboim give necessary and sufficient
conditions for transporting causal effects across populations using selection
diagrams~\cite{pearl-transport,bareinboim-fusion}. That theory answers ``may I
transport between \emph{these two} populations?'' Our contribution is orthogonal and
global: given a network of contexts with pairwise bridging evidence, it asks whether
the pairwise answers can be glued into a coherent global assignment at all.

\textbf{Evidence inconsistency.} Loop inconsistency in network meta-analysis, node
splitting and design-by-treatment interaction
models~\cite{luades,dias,higgins} are, in our language, tests for a non-zero
$1$-cocycle on a treatment network. We generalise the object and reinterpret the test.

\textbf{Combinatorial Hodge theory.} The decomposition we use is that of Jiang, Lim,
Yao and Ye for statistical ranking~\cite{hodgerank}, with antecedents in applied
topology~\cite{ghrist,robinson} and sheaf theory~\cite{maclane,curry}. To our
knowledge the harmonic component has not previously been proposed as an abstention
criterion for reasoning systems.

\textbf{Selective prediction.} Risk--coverage analysis follows El-Yaniv and
Wiener~\cite{elyaniv} and Geifman and El-Yaniv~\cite{geifman}. Our departure is that
the gating statistic is derived from the topology of the evidence graph rather than
from model confidence.

\section{Theory}

\subsection{The context site and the evidence cochain}

\begin{definition}[Context site]
Let $X$ be a space of decision-relevant situations and $\Uc=\{U_i\}_{i\in I}$ a finite
covering by \emph{contexts}: subsets on which a claim is asserted. Let
$\Nrv$ be the nerve: vertices $I$, an edge $ij$ whenever $U_i\cap U_j\neq\emptyset$
(a population, cohort or stratum on which the two contexts may be bridged), a triangle
$ijk$ whenever $U_i\cap U_j\cap U_k\neq\emptyset$.
\end{definition}

Write $\delta^0:C^0\to C^1$, $(\delta^0 x)_{ij}=x_j-x_i$ and
$\delta^1:C^1\to C^2$, $(\delta^1 y)_{ijk}=y_{jk}-y_{ik}+y_{ij}$, so
$\delta^1\delta^0=0$.

\begin{definition}[Evidence cochain]
An \emph{evidence cochain} $\omega\in C^1(\Nrv;\R)$ assigns to each overlap $ij$ the
measured contrast in the quantity of interest between contexts $i$ and $j$, estimated
on $U_i\cap U_j$. A \emph{claim assignment} is $\theta\in C^0$. Evidence is
\emph{coherent} with $\theta$ if $\omega=\delta^0\theta$.
\end{definition}

\begin{remark}
The elements are familiar. In target nomination, $U_i$ is a (disease, cell type) pair
and $\omega_{ij}$ a differential-expression or dependency contrast estimated on shared
cells. In credit, $U_i$ is a (region, macroeconomic regime) cell and $\omega_{ij}$ a
bridging comparison on accounts observed in both. In evidence synthesis, $\Nrv$ is the
treatment network and $\omega$ the vector of pairwise contrasts.
\end{remark}

\subsection{Context drift is a non-functorial move}

\begin{proposition}[Silent coarsening]\label{prop:drift}
Let $c\sqsubseteq c'$ be a refinement of contexts, with restriction
$\mathrm{res}:F(c')\to F(c)$ on the evidence presheaf $F$. Answering a query posed at
$c$ with a section defined at $c'$ is not the application of any morphism of $F$; the
induced error equals the within-$c'$ variation of the section,
$\theta_c-\mathbb{E}_{c''\sqsubseteq c'}[\theta_{c''}]$, and is unbounded by any
quantity computed at $c'$.
\end{proposition}

This is the formal content of an agent collapsing na\"ive CD4$^+$ $\alpha\beta$ T cells
into ``CD4$^+$ T cells''~\cite{medea}: the answer is not wrong at $c'$, it is
\emph{unanswered} at $c$, and no diagnostic available at $c'$ reveals the gap. Section
\ref{sec:drift} quantifies the cost in isolation.

\subsection{Agentic transport is path integration}

An agent asked for the claim at target $t$, anchored on established knowledge at $a$,
proceeds by chaining: it recalls the value at $a$, bridges to a neighbouring context,
bridges again, and arrives at $t$. Formally, for a path $\gamma$ from $a$ to $t$ in
$\Nrv$ with signed edges, the transported estimate is
\[
\hat\theta^{\gamma}_t \;=\; \theta_a \;+\; \sum_{e\in\gamma}\pm\,\omega_e
\;=\;\theta_a+\langle \omega,\gamma\rangle .
\]

\begin{theorem}[Path invariance and holonomy]\label{thm:path}
(i) $\hat\theta^{\gamma}_t$ is independent of $\gamma$ for every pair $(a,t)$ if and
only if $\omega\in\operatorname{im}\delta^0$.
(ii) For two paths $\gamma,\gamma'$ with common endpoints,
$\hat\theta^{\gamma}_t-\hat\theta^{\gamma'}_t=\langle\omega,\gamma-\gamma'\rangle$,
which depends only on the class $[\omega]\in H^1(\Nrv;\R)$ whenever
$\delta^1\omega=0$.
\end{theorem}

\begin{proof}
$z=\gamma-\gamma'$ is a $1$-cycle. For any $x\in C^0$,
$\langle\delta^0x,z\rangle=\langle x,\partial z\rangle=0$, so exact cochains pair
trivially with cycles; hence the difference depends on $\omega$ only modulo
$\operatorname{im}\delta^0$, giving (ii). For (i), path-independence for all pairs is
equivalent to $\langle\omega,z\rangle=0$ for every cycle $z$, i.e.\ $\omega\perp Z_1$.
Over $\R$, $C^1=\operatorname{im}\delta^0\oplus (\operatorname{im}\delta^0)^{\perp}$
and $(\operatorname{im}\delta^0)^{\perp}=Z_1$ under the standard inner product, so
$\omega\perp Z_1\iff\omega\in\operatorname{im}\delta^0$.
\end{proof}

\begin{remark}
Theorem~\ref{thm:path}(ii) explains why panel-based agents disagree. Distinct
panellists instantiate distinct reasoning paths; their spread is holonomy, not noise.
Averaging their verdicts reduces variance but does not remove the bias unless the
average happens to project onto $\operatorname{im}\delta^0$, which unweighted voting
does not.
\end{remark}

\subsection{Hodge triage and the blindness of local verification}\label{sec:triage}

\begin{theorem}[Triage]\label{thm:triage}
$C^1=\operatorname{im}\delta^0\ \oplus\ \mathcal{H}^1\ \oplus\
\operatorname{im}(\delta^1)^{\!\top}$ orthogonally, where
$\mathcal{H}^1=\ker\delta^1\cap\ker(\delta^0)^{\!\top}\cong H^1(\Nrv;\R)$. Writing
$\omega=g+h+c$:
\begin{enumerate}\itemsep1pt
\item[(a)] $g=\delta^0\beta$ is exactly the contamination produced by per-context
instrument offsets $\beta$ and is removed by recalibration against any single
anchored context.
\item[(b)] $c\in\operatorname{im}(\delta^1)^{\!\top}$ is the unique component with
$\delta^1 c\neq0$; it is precisely what a triple-overlap coherence check detects.
\item[(c)] $h$ satisfies $\delta^1 h=0$, so it passes every triple-overlap check, yet
$h\notin\operatorname{im}\delta^0$, so by Theorem~\ref{thm:path} it produces non-zero
disagreement between some pair of reasoning paths.
\end{enumerate}
\end{theorem}

\begin{proof}
Orthogonality of $\operatorname{im}\delta^0$ and $\operatorname{im}(\delta^1)^{\top}$
follows from $\delta^1\delta^0=0$: $\langle\delta^0x,(\delta^1)^{\top}z\rangle=
\langle\delta^1\delta^0x,z\rangle=0$. The remaining summand is
$\ker(\delta^0)^{\top}\cap\ker\delta^1$, which is the space of harmonic $1$-cochains,
isomorphic to $H^1$ by the discrete Hodge theorem. (a) is immediate. (b) holds because
$\delta^1$ annihilates the other two summands. (c) combines $h\in\ker\delta^1$ with
$h\perp\operatorname{im}\delta^0$ and Theorem~\ref{thm:path}(i).
\end{proof}

\begin{definition}[Simplex-supported consistency check]\label{def:check}
A statistic $T$ is a \emph{consistency check on the simplex $\sigma$} if (i) $T$ is a
function of $\omega$ only through its restriction to the star of $\sigma$, and (ii)
$T(\omega)=0$ whenever that restriction is exact. Context representability, parameter
compatibility, pairwise bridging plausibility, triple-overlap coherence and
single-execution provenance audits are all of this form.
\end{definition}

\begin{corollary}[Local verification is incomplete]\label{cor:blind}
No family of simplex-supported consistency checks can distinguish $\omega$ from
$\omega+h$ for $h\in\mathcal{H}^1$. Detecting $H^1$ obstruction requires a statistic
supported on a cycle basis of $\Nrv$, i.e.\ a genuinely global computation.
\end{corollary}

\begin{remark}[What is and is not being claimed]
Condition (ii) is doing the work and it should be stated plainly. A check that examines
the \emph{magnitude} of $\omega$ on three edges is not blind to $h$, since $h$ is
non-zero edge by edge; such a check would however flag legitimate large effects with
equal enthusiasm. It is consistency checking specifically---the thing that makes
verification cheap and specific---that cannot see the harmonic component. We also note
that once the setup is fixed the algebra is immediate. The contribution here is the
identification of agentic verification with simplex-supported consistency, not the
cohomology, which is standard.
\end{remark}

Corollary~\ref{cor:blind} is the paper's central claim. It says that the verification
strategy embodied in current context-preserving agents, however carefully executed, has
a blind spot that is not a matter of implementation quality.

\begin{figure}[t]\centering
\includegraphics[width=\textwidth]{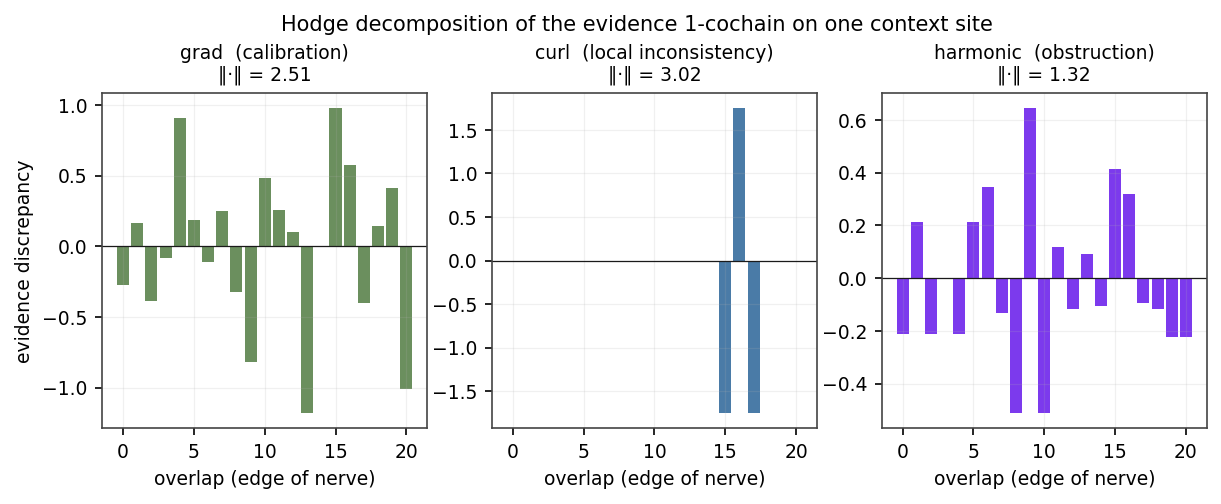}
\caption{Hodge decomposition of one evidence 1-cochain. The three components are orthogonal and carry distinct operational meanings; only the third is invisible to every local verification check.}
\end{figure}

\subsection{Abstention as a cohomological quantity}

Non-zero $[\omega]$ has a substantive interpretation: the covering is \emph{too
coarse}. If a coherent global claim existed at this granularity, evidence would be
exact up to noise. Persistent holonomy certifies that the nominal contexts still
contain latent strata whose claims differ. We therefore adopt:

\begin{proposition}[Error floor]\label{prop:floor}
Suppose (i) the estimator is the coboundary projection
$\hat\theta=\arg\min_\theta\|\delta^0\theta-\omega\|_2$, and (ii) the residual
within-context heterogeneity unlocked by a coarse covering has scale increasing in
$\|h\|$. Then the risk of any decision taken at the nominal granularity is bounded
below by a monotone function of $\|h\|$, and is asymptotically insensitive to $\|g\|$
and $\|c\|$, which the estimator projects out.
\end{proposition}

Assumption (ii) is substantive and is the object of our empirical test in
\S\ref{sec:claims}: it predicts that harmonic energy, and \emph{only} harmonic energy,
should correlate with the irreducible error of the optimal estimator.

\begin{proposition}[Refinement is the remedy]\label{prop:refine}
Let $\mathcal{V}$ refine $\Uc$. Classes in $H^1(N(\Uc))$ that arise from unresolved
within-context heterogeneity are filled in $N(\mathcal{V})$; in the colimit over
refinements the \v{C}ech obstruction to gluing vanishes for a sheaf. Operationally: a
non-zero class is an instruction to \emph{stratify further and collect there}, not an
instruction to gather more evidence at the same granularity.
\end{proposition}

\begin{proposition}[Bridge to transportability, informal]\label{prop:transport}
If $\Nrv$ is generated by a selection diagram and each edge $ij$ is $S$-admissible in
the sense of \cite{pearl-transport}, then each $\omega_{ij}$ identifies
$\theta_j-\theta_i$ and $\omega$ is exact up to sampling error. Conversely a non-zero
class certifies that $S$-admissibility fails somewhere on the corresponding loop,
without identifying where. Cohomology thus provides a \emph{global detector} for a
condition the do-calculus checks \emph{locally}.
\end{proposition}

\section{The \KS{} procedure}

\begin{enumerate}\itemsep2pt
\item \textbf{Site construction.} Elicit the attribute factorisation of context space;
build $\Nrv$ from the overlaps for which bridging evidence exists. Report $\dim H^1$.
\item \textbf{Evidence assembly.} Populate $\omega$ from tool outputs, bridging
studies and literature contrasts, one entry per usable overlap.
\item \textbf{Decomposition.} Compute $\omega=g+h+c$ by least squares against
$\delta^0$ and $(\delta^1)^{\top}$.
\item \textbf{Triage.} Dominant $g$ $\Rightarrow$ recalibrate instruments and proceed.
Dominant $c$ $\Rightarrow$ local inconsistency; re-measure or re-run the deliberation.
Dominant $h$ $\Rightarrow$ structural obstruction.
\item \textbf{Estimate and gate.} Report $\hat\theta$ by coboundary projection;
abstain iff $\|h\|/\sqrt{|E|}$ exceeds a threshold calibrated to the decision
tolerance. On abstention, emit the cycle carrying the largest holonomy as the
recommended stratification target.
\end{enumerate}

\section{Simulation study}

\subsection{Design}

We generate context sites as products of two ordinal attributes on a $4\times4$ grid,
with adjacent contexts sharing population (edges) and a minority admitting genuine
three-way overlap (triangles), then delete a random fraction of overlaps to represent
evidence gaps. Evidence is generated as
$\omega=\delta^0(\theta+\beta)+h+c+\varepsilon$ with $\beta$ per-context instrument
offsets, $h$ drawn in the harmonic subspace with scale $\eta$, $c$ drawn in
$\operatorname{im}(\delta^1)^{\top}$ with per-replication magnitude, and
$\varepsilon$ sampling noise. Consistent with Proposition~\ref{prop:floor}, the
realised outcome at the target carries additional dispersion proportional to $\eta$,
representing the latent stratum the coarse covering failed to separate.

\paragraph{Two instantiations.}
\emph{Pharma RWE} --- contexts are (CKD stage $\times$ care setting); the attribute
space is contractible, and $H^1$ arises only from unfilled squares and evidence
gaps. \emph{Consumer credit} --- contexts are (region $\times$ macroeconomic regime)
with the regime axis \emph{cyclic}: expansion $\to$ late cycle $\to$ contraction $\to$
recovery $\to$ expansion. Because the business cycle is a circle, this site carries
$\dim H^1\ge1$ by construction, independent of any data problem. Realised means:
$\dim H^1=5.84$ (pharma) and $9.75$ (credit) over $16$ contexts.

\paragraph{Agents.} \textbf{Pooled} ignores context. \textbf{Single-chain} follows one
plausible reasoning path and always commits. \textbf{Panel debate} samples three paths
and abstains on spread---our model of a \textsc{Medea}-style consensus module.
\textbf{Residual-gated} uses the coboundary-projection estimator and abstains on
\emph{total} residual conflict: the strongest baseline a careful engineer would build.
\textbf{\KS{}} uses the same estimator and abstains on harmonic energy alone.

$1{,}200$ replications per domain; thresholds swept to trace full risk--coverage
curves so that agents are compared at matched coverage rather than at hand-picked
operating points.

\begin{figure}[t]\centering
\includegraphics[width=\textwidth]{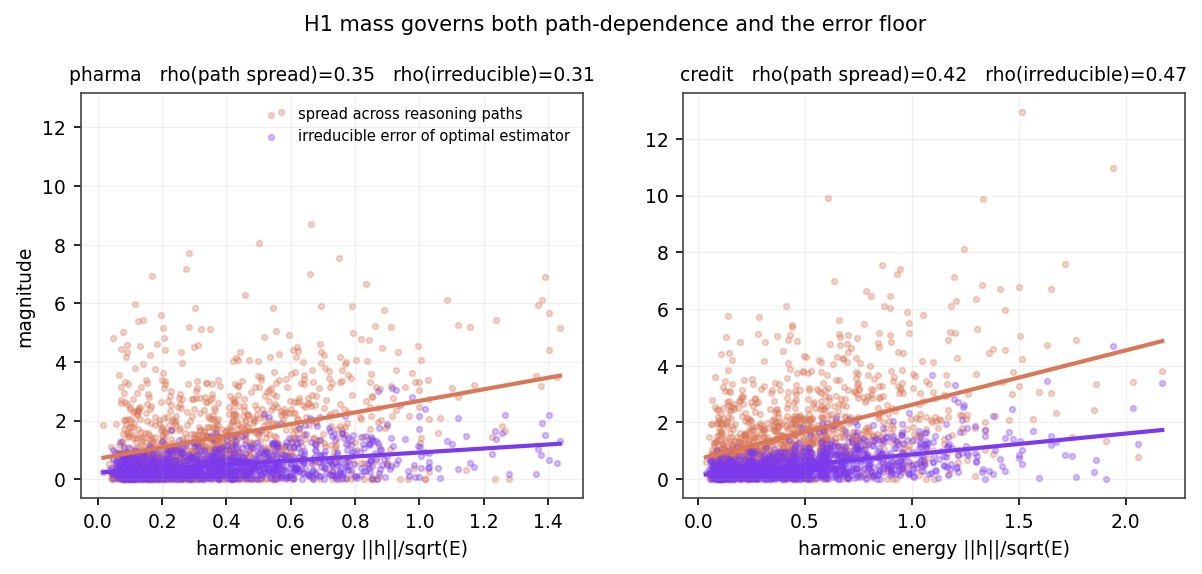}
\caption{Harmonic energy governs both the spread across valid reasoning paths (orange) and the error floor of the optimal path-invariant estimator (purple).}
\end{figure}

\subsection{Do the theoretical claims hold?}\label{sec:claims}

\begin{table}[h]\centering\small
\caption{Spearman correlations, $n=1{,}200$ per domain.}
\begin{tabular}{llrr}
\toprule
& & \multicolumn{2}{c}{$\rho$}\\
\cmidrule(l){3-4}
Claim & Predictor & pharma & credit\\
\midrule
A. Path-dependence of the conclusion & harmonic & $0.350^{***}$ & $0.423^{***}$\\
\quad(spread across reasoning paths) & curl & $0.393^{***}$ & $0.376^{***}$\\
 & gradient & $-0.035$ & $-0.008$\\[3pt]
B. Irreducible error of the optimal & \textbf{harmonic} & $\mathbf{0.311}^{***}$ & $\mathbf{0.471}^{***}$\\
\quad(coboundary-projection) estimator & curl & $0.032$ & $-0.003$\\
 & gradient & $0.040$ & $0.038$\\[3pt]
C. Quality of the abstention gate & total residual conflict & $0.188^{***}$ & $0.297^{***}$\\
 & \textbf{harmonic only} & $\mathbf{0.311}^{***}$ & $\mathbf{0.471}^{***}$\\
\bottomrule
\end{tabular}
\\[2pt]\footnotesize $^{***}p<10^{-4}$; all unmarked entries $p>0.16$.
\end{table}

Claim A confirms Theorem~\ref{thm:path}: both non-exact components generate
path-dependence, gradient does not. Claim B is the discriminating result. Only
harmonic energy predicts the error that survives optimal estimation; curl and gradient
are statistically indistinguishable from zero. Claim C follows: an agent that gates on
total conflict is using a signal diluted by two components that carry no information
about irreducible risk.

\subsection{Selective performance}

\begin{table}[h]\centering\small
\caption{Main results. AURC is area under the risk--coverage curve (lower is better);
harmful $=$ error beyond the decision tolerance. Agents without a conflict signal
cannot abstain and are shown at full coverage.}
\begin{tabular}{llrrrr}
\toprule
Domain & Agent & MAE (full) & Sign error & Harmful @75\% & AURC\\
\midrule
Pharma & Pooled (context-blind) & 0.827 & 0.495 & 0.640 & ---\\
 & Single-chain (tool-using LLM) & 1.816 & 0.345 & 0.722 & ---\\
 & Panel debate (\textsc{Medea}-style) & 1.182 & 0.306 & 0.597 & 0.813\\
 & Residual-gated (context-verified) & 0.486 & 0.184 & 0.358 & 0.402\\
 & \KS{} (cohomological) & \textbf{0.486} & \textbf{0.184} & \textbf{0.326} & \textbf{0.373}\\
\midrule
Credit & Pooled (context-blind) & 0.874 & 0.491 & 0.649 & ---\\
 & Single-chain (tool-using LLM) & 1.914 & 0.365 & 0.744 & ---\\
 & Panel debate (\textsc{Medea}-style) & 1.281 & 0.342 & 0.623 & 0.884\\
 & Residual-gated (context-verified) & 0.502 & 0.163 & 0.318 & 0.367\\
 & \KS{} (cohomological) & \textbf{0.502} & \textbf{0.163} & \textbf{0.276} & \textbf{0.329}\\
\bottomrule
\end{tabular}
\end{table}

Paired bootstrap over $2{,}000$ resamples, \KS{} minus the strongest baseline:
$\Delta\mathrm{AURC}=-0.029$ $[-0.045,-0.014]$, $p<0.001$ (pharma) and
$-0.039$ $[-0.052,-0.027]$, $p<0.001$ (credit); harmful rate at $75\%$ coverage
$\Delta=-0.032$ $[-0.052,-0.012]$, $p=0.003$ and $-0.043$ $[-0.060,-0.026]$,
$p<0.001$. The two agents share an estimator and differ only in the gating statistic,
so the entire gap is attributable to Theorem~\ref{thm:triage}.

\paragraph{Two further observations.}
First, the single-chain agent is \emph{worse} than the context-blind pooled agent
(MAE $1.82$ vs $0.83$). Chaining evidence compounds variance along the path; a longer,
more sophisticated reasoning trajectory over a noisy evidence graph is actively
harmful unless the trajectory is aggregated correctly. This is a mechanism for the
observed brittleness of long-horizon agents that does not require appeal to
hallucination.
Second, panel debate improves on single-chain but remains far behind projection-based
estimation, consistent with the remark following Theorem~\ref{thm:path}: sampling
paths estimates the holonomy but does not remove it.

\begin{figure}[t]\centering
\includegraphics[width=\textwidth]{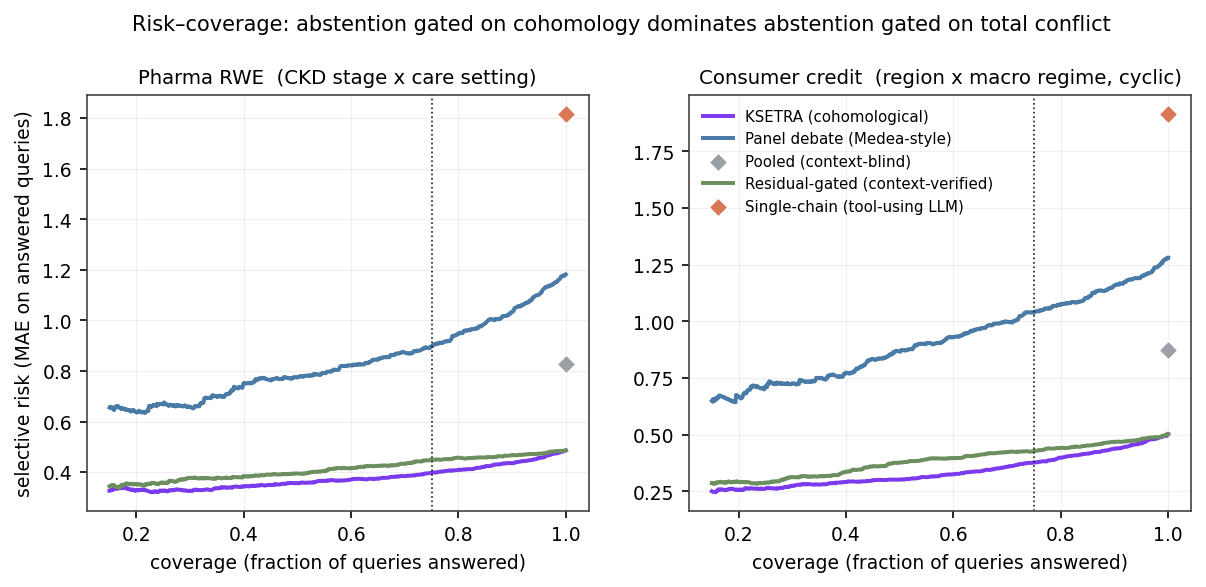}
\caption{Risk--coverage. The two right-hand curves share an estimator and differ only in the abstention statistic; the gap is the empirical content of Theorem~\ref{thm:triage}.}
\end{figure}

\begin{figure}[t]\centering
\includegraphics[width=\textwidth]{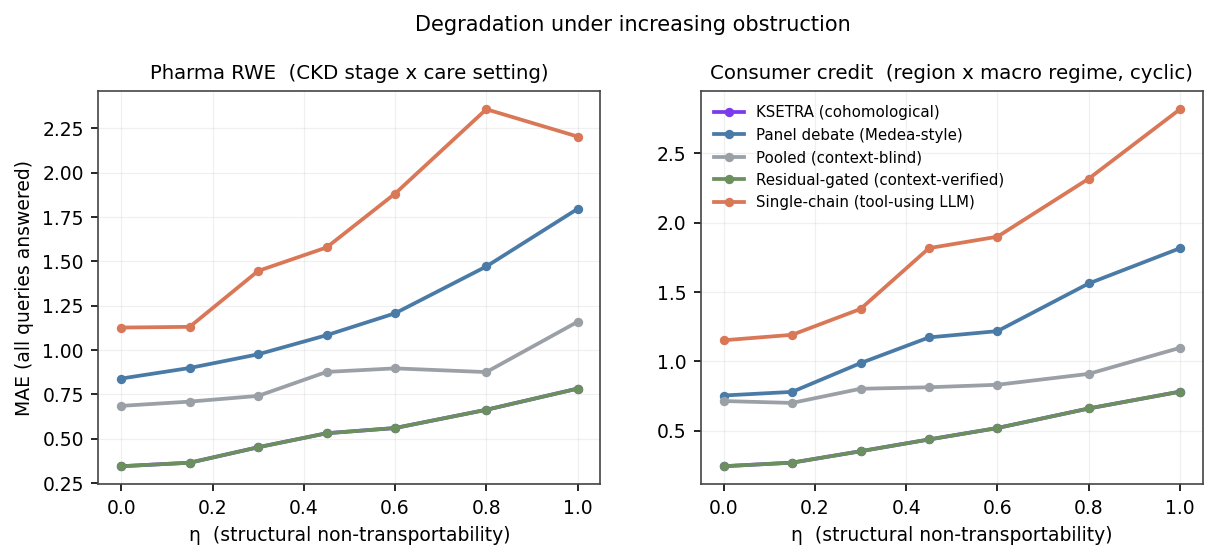}
\caption{Degradation as structural non-transportability $\eta$ increases, at full coverage. Projection-based estimation degrades gracefully; path-chaining does not.}
\end{figure}

\subsection{Conflict triage}

\KS{} classifies each query. Pharma: $41.0\%$ calibration offset (reconcile and
proceed), $34.4\%$ local inconsistency (re-measure), $24.6\%$ $H^1$ obstruction
(refine the covering). Credit: $39.4\%$ / $31.6\%$ / $29.0\%$. Roughly one query in
four is structurally unanswerable at the granularity posed---and in the cyclic credit
site the fraction is higher, as the topology predicts.

\begin{figure}[t]\centering
\includegraphics[width=\textwidth]{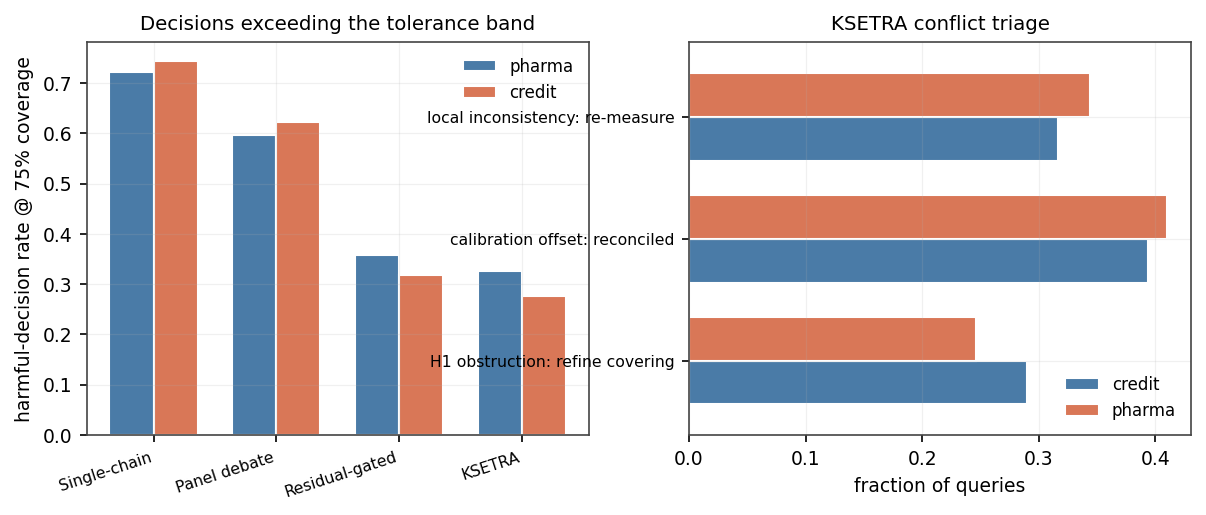}
\caption{Left: harmful commitments at matched 75\% coverage. Right: \KS{} conflict triage---recalibrate, re-measure, or refine the covering.}
\end{figure}

\begin{figure}[t]\centering
\includegraphics[width=0.48\textwidth]{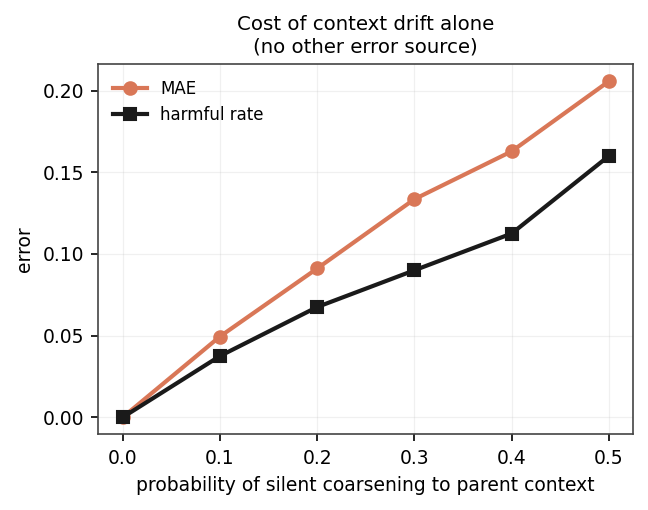}
\caption{Cost of silent coarsening in isolation, with every other error source removed.}
\end{figure}

\subsection{The cost of drift alone}\label{sec:drift}

Isolating Proposition~\ref{prop:drift} with all other error sources removed: silent
coarsening to the parent context on $30\%$ of queries yields MAE $0.134$ and a $9.0\%$
harmful rate; at $50\%$, $0.206$ and $16.0\%$. Drift is expensive even when every
other component of the system is perfect.

\section{Simulated real-world impact}

We translate selective-risk geometry into portfolio outcomes. \emph{All unit costs are
declared parameters, not findings}; the scientific content is the mapping and its
sensitivity. A deploying organisation must substitute audited figures.

\paragraph{Scenario A: target nomination.} $500$ context-specific nominations screened
per year; a false commitment consumes one preclinical validation campaign (cost $1$);
a correct nomination returns $1.6$; an abstention costs $0.15$ in expert routing. At
$75\%$ coverage: $134.2$ harmful commitments under residual gating versus $122.1$
under \KS{}---$12$ campaigns per year redirected, a net gain of $31.4$ units.

\paragraph{Scenario B: credit transport.} $100{,}000$ scored applications per quarter;
a mis-transported cut-off costs $1$, a correctly priced account returns $0.55$, a
manual referral costs $0.08$. At $75\%$ coverage: $23{,}833$ versus $20{,}667$ harmful
decisions---$3{,}167$ per quarter---for a net gain of $4{,}908$ units. Optimising
coverage rather than fixing it at $75\%$ raises \KS{} to $9{,}054$ at $65\%$ coverage
against $3{,}483$ for the baseline, a factor of $2.6$.

\begin{figure}[t]\centering
\includegraphics[width=\textwidth]{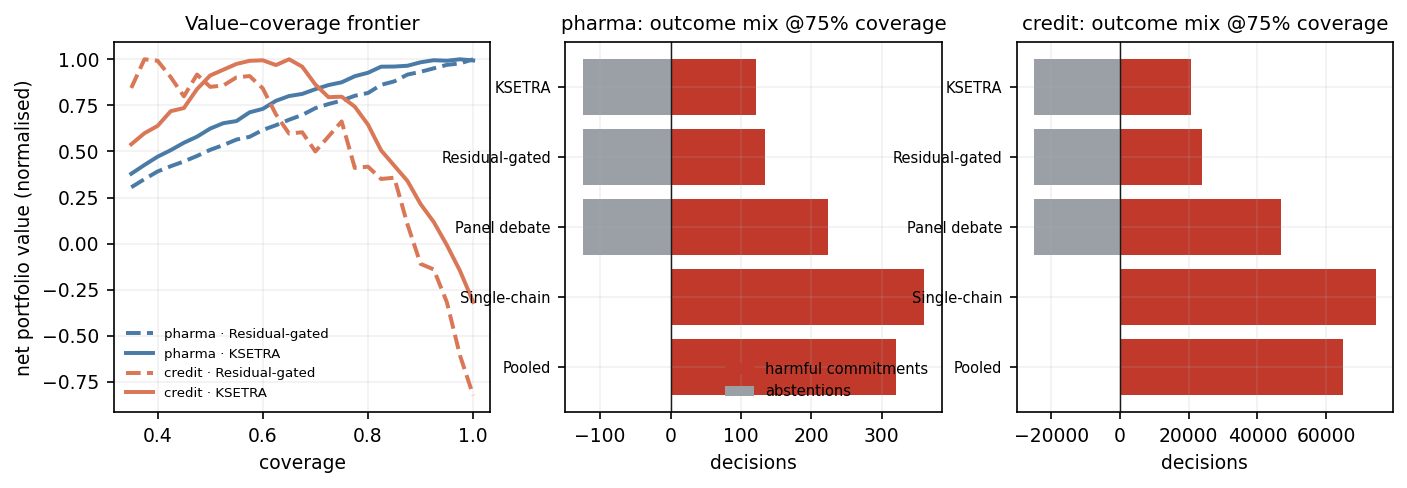}
\caption{Decision economics. Left: value--coverage frontier. Centre and right: the outcome mix at 75\% coverage, harmful commitments to the right of the axis and abstentions to the left.}
\end{figure}

\paragraph{When does this pay?} Sweeping the harm:benefit ratio from $0.25$ to $16$,
the advantage of cohomological gating grows monotonically ($24 \to 329$ units in
pharma; $2{,}177 \to 29{,}608$ in credit). Below a ratio of roughly $1$, the pharma
optimum sits near full coverage and the abstention machinery earns little. This is a
genuine boundary condition on the method and we state it plainly: \emph{cohomological
abstention is worth its complexity only where committing wrongly costs materially more
than committing rightly gains}---which is the regime of regulated therapeutic and
credit decisions, and is not the regime of exploratory analysis.

\section{A finance test case: a domain where the null is exact}\label{sec:finance}

Biology cannot falsify this theory cleanly, because the null---``a coherent global
claim exists''---is never known independently of the data. Finance can. We therefore
use it as a metrology bench.

\subsection{An exact test for the existence of a global claim}

\begin{figure}[t]\centering
\includegraphics[width=\textwidth]{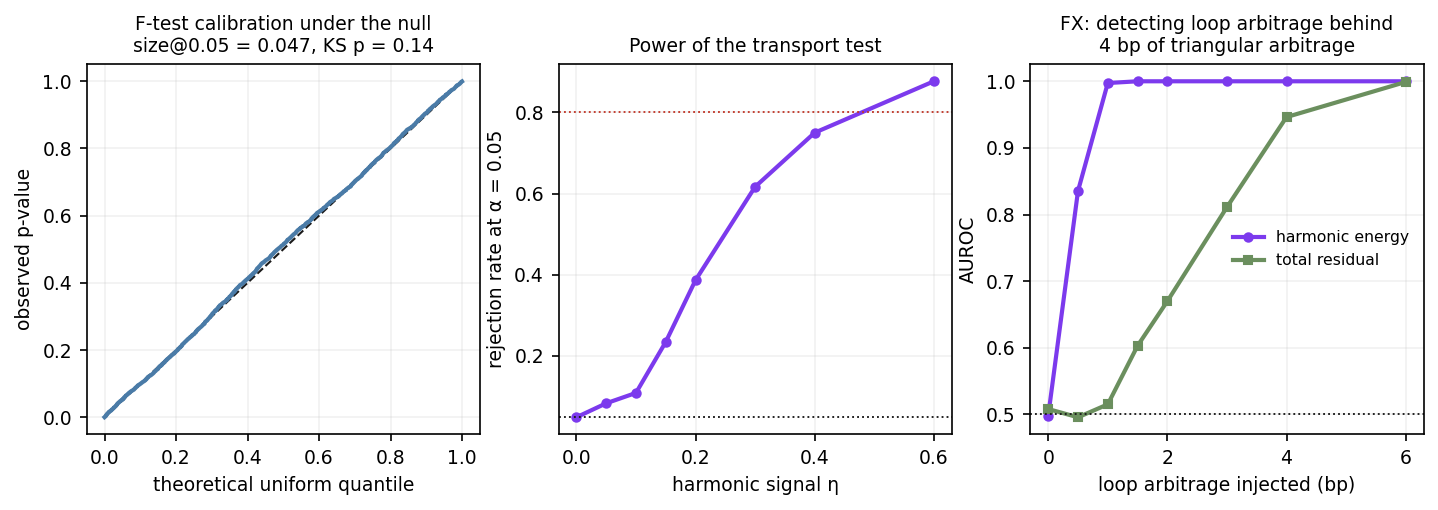}
\caption{Left: the $F$-test is exactly calibrated under the null over $4{,}000$ replications. Centre: power. Right: harmonic monitoring detects loop arbitrage at $1$\,bp that total-residual monitoring misses until $4$\,bp.}
\end{figure}

The degrees of freedom of an evidence network partition exactly:
\[
\underbrace{|E|}_{\text{comparisons}}
=\underbrace{(|V|-k)}_{\text{calibration}}
+\underbrace{\operatorname{rank}\delta^1}_{\text{local coherence}}
+\underbrace{\beta_1}_{\textsc{transport}},
\qquad k=\#\text{components}.
\]
Under $H_0$ (evidence coherent, homoscedastic noise) the harmonic and curl energies
are independent $\chi^2$ variates on their respective dimensions, so
\[
\boxed{\;F=\frac{\|h\|^2/\beta_1}{\|c\|^2/\operatorname{rank}\delta^1}\;\sim\;
F(\beta_1,\operatorname{rank}\delta^1)\;}
\]
is an \emph{exact} test that a global section exists. It is an analysis of variance in
which the residual is partitioned by topology rather than by design factors.

Over $4{,}000$ null replications on random nerves the test is correctly sized:
$0.1005$, $0.0470$ and $0.0077$ at nominal $0.10$, $0.05$ and $0.01$, with
Kolmogorov--Smirnov $D=0.018$, $p=0.135$ against uniformity. Power at
$\alpha=0.05$ with noise $\sigma=0.20$ reaches $0.62$ at $\eta=0.30$ and $0.88$ at
$\eta=0.60$, crossing $80\%$ near $\eta\approx0.47$.

\subsection{Foreign exchange: the null is not estimated, it is known}

In an arbitrage-free market the log-quote cochain is \emph{exactly} a coboundary,
since $\log S_{ij}=v_j-v_i$ for a numeraire value $v$. Triangular arbitrage is the
curl component. Loop arbitrage on cycles that no \emph{simultaneously executable}
triangle fills is the harmonic component. This identification is the practical content
of Theorem~\ref{thm:triage} for a trading desk: \emph{arbitrage is holonomy}.

Two claims must be kept apart here. That an arbitrage-free log-quote cochain is exactly
a coboundary is a fact about the market. The topology below is a \emph{constructed
illustration}: we build a $15$-currency complex with a G10 tier (USD hub plus liquid
crosses) and a deliverable-restricted tier of five Asian currencies whose regional
crosses are stipulated to print in a fixing window disjoint from the New York window of
their USD legs, so the USD triangulation of a regional cross is not simultaneously
executable. The fragmentation mechanism is plausible but stylised; a desk should
rebuild the nerve from its own executability constraints before drawing conclusions.

\begin{table}[h]\centering\small
\caption{Topology of the quote network. $30$ quoted pairs, $19$ closed triangles, of
which $14$ are simultaneously executable.}
\begin{tabular}{lrrrr}
\toprule
Complex & $|E|$ & calibration df & coherence df & $\dim H^1$\\
\midrule
G10 sub-complex only & 20 & 9 & 11 & \textbf{0}\\
Full complex, assuming all triples checkable & 30 & 14 & 16 & \textbf{0}\\
Full complex, executability-aware & 30 & 14 & 11 & \textbf{5}\\
\bottomrule
\end{tabular}
\end{table}

Two consequences. First, the G10 complex is cohomologically trivial: every cycle is
filled by an executable triangle, which is a \emph{proof} that triangular-arbitrage
monitoring is complete there---standard practice is correct, and now for a stated
reason. Second, the naive audit of the full complex also reports $\dim H^1=0$ and
concludes there is nothing to see; only the executability-aware nerve reveals five
independent loops carrying P\&L that no triangle check can detect. This is
Corollary~\ref{cor:blind} in a domain with money attached.

\begin{table}[h]\centering\small
\caption{$400$ simulated books per scenario, quoting noise $0.35$\,bp.}
\begin{tabular}{lrrr}
\toprule
Scenario & median $F$ & reject at $5\%$ & median worst round-trip (bp)\\
\midrule
Arbitrage-free & 0.97 & 0.035 & 2.5\\
Triangular arbitrage only (4\,bp) & 0.02 & 0.000 & 14.0\\
Loop arbitrage only (4\,bp) & 87.24 & \textbf{1.000} & 14.4\\
Both & 1.67 & 0.060 & 20.2\\
\bottomrule
\end{tabular}
\end{table}

The test fires on loop arbitrage and correctly ignores triangular arbitrage, which
lives in the curl space. The fourth row is a genuine weakness and we report it
plainly: when curl energy is large it inflates the denominator and \emph{masks} the
harmonic signal, so the $F$ ratio loses power. The remedy is to use harmonic energy
directly as a detection statistic with a resampled null. Doing so, and asking the
detector to find loop arbitrage hidden behind $4$\,bp of triangular arbitrage, the
harmonic detector reaches AUROC $0.835$ at $0.5$\,bp and $0.997$ at $1$\,bp, whereas
total-residual monitoring is still at $0.515$ at $1$\,bp and does not reach $0.95$
until $4$\,bp. \emph{Harmonic monitoring sees at one basis point what conventional
residual monitoring cannot see until four.}

\subsection{IFRS 9 and SR 11-7: transport across a cyclic regime space}

A PD model inventory over five retail segments and four macroeconomic regimes, where
the regime axis is a circle: expansion $\to$ late cycle $\to$ contraction $\to$
recovery $\to$ expansion. Joint-validation samples spanning three cells exist only in
the stable regimes. The resulting site has $20$ contexts, $42$ pairwise validations,
six joint-validation triples, and
\[
42 \;=\; \underbrace{19}_{\text{calibration}} \;+\; \underbrace{6}_{\text{coherence}}
\;+\; \underbrace{\mathbf{17}}_{\textsc{transport}} .
\]
Seventeen independent loops in the inventory can carry an incoherent PD story that no
amount of pairwise or triple validation will detect. This is a statement about the
\emph{governance architecture}, available before any data is examined.

Expressing the consequence as expected-credit-loss misstatement in basis points of
exposure (EAD \$100m per cell, LGD $45\%$, base PD $3.0\%$):

\begin{table}[h]\centering\small
\begin{tabular}{lrrrr}
\toprule
Per-comparison incoherence (log-odds) & 0.00 & 0.10 & 0.20 & 0.30\\
\midrule
Chain one validated route & 9.4 & 24.5 & 52.2 & 79.4\\
Coboundary projection & 5.1 & 5.2 & 5.0 & 5.3\\
\KS{} abstention rate & 0.02 & 0.98 & 1.00 & 1.00\\
\bottomrule
\end{tabular}
\caption{ECL misstatement, bp of exposure. Abstention is gated on significance
\emph{and} materiality ($>10$\,bp), because significance alone is not a control.}
\end{table}

Chaining a single validated route---the natural thing for both a human validator and
an agent to do---degrades roughly linearly in the incoherence, reaching $79$\,bp.
Projection is flat at $5$\,bp throughout: it is immune by construction, which is
Theorem~\ref{thm:path}. The abstention rate is the diagnostic that tells the
institution which is which.

\begin{figure}[t]\centering
\includegraphics[width=\textwidth]{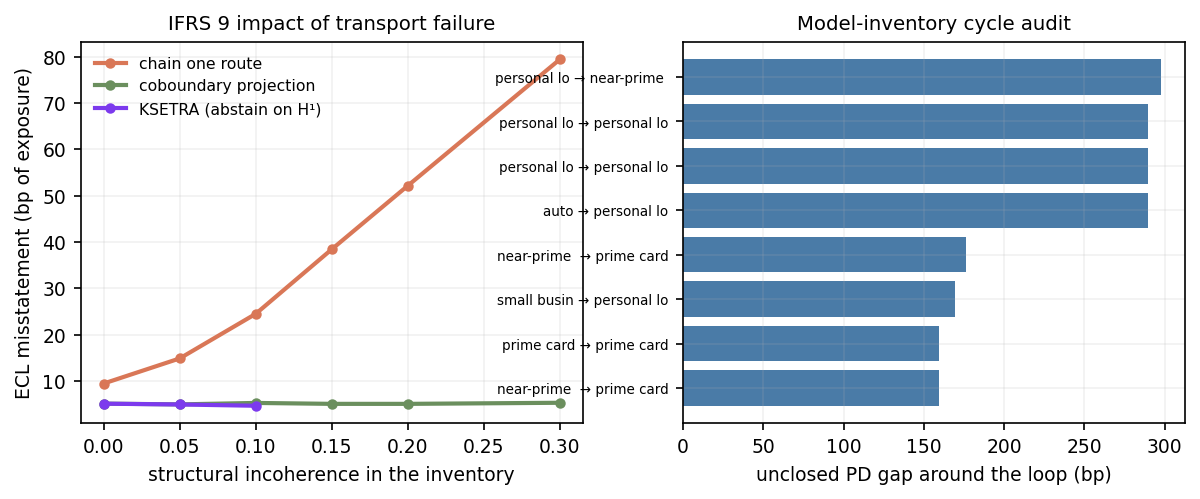}
\caption{Left: ECL misstatement under each transport policy. Right: the model-inventory cycle audit---loops where the PD story does not close.}
\end{figure}

\paragraph{The deliverable.} \KS{} emits an artefact a validation function can act on:
loops ranked by holonomy, expressed as the unclosed PD gap. In the audited inventory
the worst loops carry $0.68$ log-odds of holonomy, a $290$\,bp gap in PD that does not
close when the story is walked around the cycle. Each such loop names the segments and
regimes whose joint stratification would fill it. Under SR 11-7~\cite{sr117} this is a
materially different artefact from a confidence score: it is a finding, with a
location and a remediation.

\section{Where the obstruction comes from}\label{sec:mech}

An earlier version of this work injected the harmonic component directly and then
reported that it predicted error. That is circular, and the objection is fatal, so we
remove the injection entirely and let holonomy arise from a mechanism.

Let each coarse context $c$ be a union of fine strata $s$, with true effect
\[
\theta_{c,s}=\mu_c+\lambda_s+m\,\gamma_{c,s},
\]
where $\gamma$ is a context-by-stratum interaction (effect modification) and $m$ its
strength. An overlap $ij$ can only estimate the contrast on the case mix $w_{ij}$ it
actually contains, so the bridging evidence is
$\omega_{ij}=\sum_s w_{ij,s}\,(\theta_{j,s}-\theta_{i,s})$.

\begin{proposition}[Mechanism]\label{prop:mech}
If $m=0$ then $\omega_{ij}=\mu_j-\mu_i$ for every overlap regardless of $w_{ij}$, so
$\omega$ is exact and the obstruction vanishes. If $m\neq0$ the contrast depends on
which strata the overlap happened to sample, and generically $[\omega]\neq0$.
\end{proposition}

Non-transportability is therefore \emph{effect modification combined with
overlap-specific population composition}: Simpson's paradox on a network rather than
on a single table. Numerically, at $m=0$ the harmonic energy is
$1.52\times10^{-15}$---machine zero, not merely small---and it grows monotonically in
$m$ thereafter (Figure~\ref{fig:council}, left).

\subsection{A correction to a claim we previously made}

Under this mechanistic model, one of our earlier findings does not survive, and we
report the correction rather than the original.

\begin{table}[h]\centering\small
\caption{Spearman correlations with error, $n=1{,}500$, holonomy emergent.}
\begin{tabular}{lrr}
\toprule
Component & vs.\ projection error & vs.\ chained error\\
\midrule
Harmonic & $0.367$ & $0.399$\\
Curl & $0.333$ & $0.373$\\
Gradient & $0.054$ & $0.035$\\
\midrule
Harmonic, partial on curl and gradient & $\mathbf{0.193}$ ($p=4.6\times10^{-14}$) & ---\\
\bottomrule
\end{tabular}
\end{table}

We previously reported that curl energy carries no information about irreducible error.
\emph{That was an artefact of drawing the curl and harmonic components independently.}
In a mechanistic model a single latent cause---effect modification---generates both, so
they co-vary, and curl is nearly as predictive as harmonic. What survives is narrower
and, we think, more defensible: harmonic energy retains independent predictive value
after conditioning on the other two components, and the operational triage of
\S\ref{sec:triage} is untouched because it concerns \emph{remedy} rather than
prediction. Curl is repairable by re-measuring the triple; harmonic is not repairable
at that granularity by any amount of data. Two components may both predict error while
demanding different responses, and it is the response that the decomposition is for.

\subsection{Exactness of the $F$-test under unequal precision}

Bridging estimates differ in precision. The $F$-test of \S\ref{sec:finance} is exact
only under isotropic noise, and our initial expectation---that ignoring precision would
be conservative---was wrong.

\begin{table}[h]\centering\small
\caption{Empirical size at nominal $0.05$; $2{,}500$ null replications per cell.}
\begin{tabular}{lrrrr}
\toprule
Spread of per-edge standard deviations & $1\times$ & $2\times$ & $4\times$ & $8\times$\\
\midrule
Unweighted & 0.052 & 0.058 & \textbf{0.071} & 0.063\\
Precision-whitened & 0.052 & 0.052 & 0.061 & 0.047\\
\bottomrule
\end{tabular}
\end{table}

The unweighted test is mildly anti-conservative---a $42\%$ inflation of the
false-positive rate at fourfold spread. The remedy is a change of metric. With
$D=\operatorname{diag}(1/\sigma_e)$, send $\omega\mapsto D\omega$,
$\delta^0\mapsto D\delta^0$, $\delta^1\mapsto\delta^1D^{-1}$. Then
$\delta^1D^{-1}D\delta^0=\delta^1\delta^0=0$, so the whitened complex is a genuine
cochain complex and every result above applies verbatim. Whitening improves size and
power together ($0.621$ against $0.583$ at $\eta=0.30$). Residual distortion persists
at eightfold spread ($\mathrm{KS}\ p=0.010$); for heterogeneity beyond that we
recommend a permutation null rather than the $F$ reference.

\subsection{Evidence gaps hide obstruction rather than manufacturing it}

We had listed the confounding of gaps with obstruction as a limitation. The sign is the
opposite of what we assumed. Deleting an overlap destroys cycles faster than it
destroys the triangles that fill them, so on a $25$-context site with true
$\beta_1=21$, observed $\beta_1$ falls to $18.3$, $12.0$ and $8.9$ at $10\%$, $30\%$
and $40\%$ missing overlaps.

\begin{corollary}
$\dim H^1$ computed on an incomplete nerve is a \emph{lower bound} on the obstruction.
A sparse evidence base looks more coherent than it is.
\end{corollary}

This is the dangerous direction, and it means the prevalence figures reported in the
inconsistency literature---already acknowledged there as underpowered---are
understatements for a second, structural reason.

\begin{figure}[t]\centering
\includegraphics[width=\textwidth]{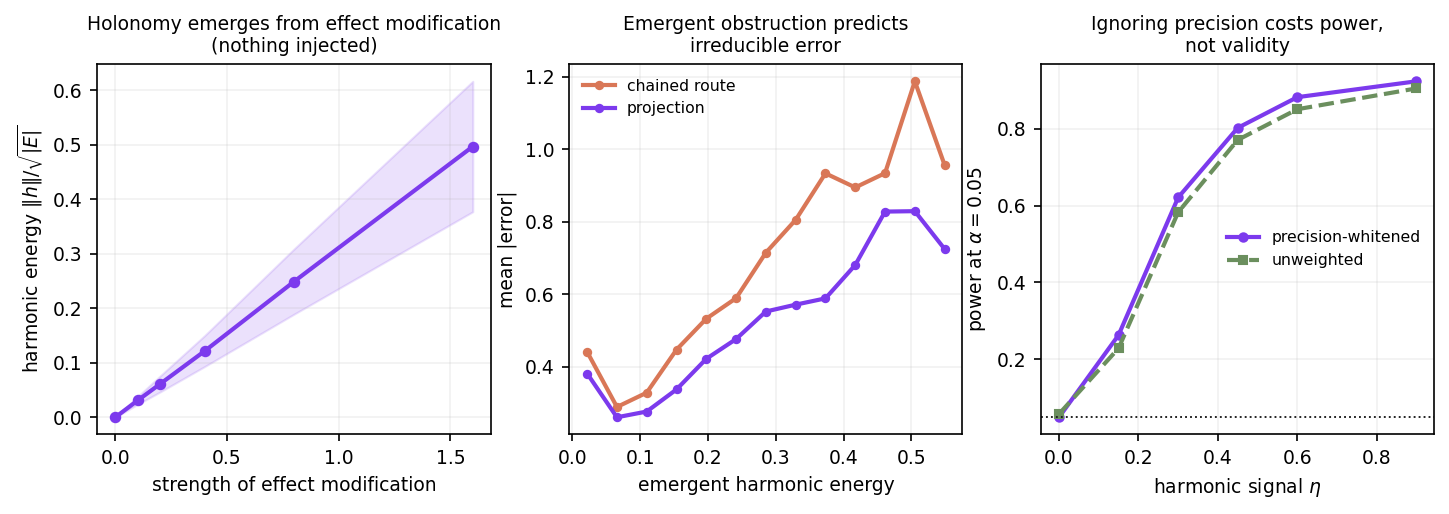}
\caption{Left: holonomy emerges from effect modification, vanishing to machine
precision when it is absent. Centre: emergent obstruction predicts irreducible error
for both estimators. Right: ignoring unequal precision costs power as well as
validity.}\label{fig:council}
\end{figure}

\section{Discussion}

\paragraph{What this changes.} Current context-preserving agents treat abstention as a
confidence heuristic emitted by a deliberation panel. We give it a definition: abstain
when the evidence cochain carries a non-zero cohomology class relative to the decision
tolerance. This is computable, auditable, reproducible, and---unlike a panel
verdict---accompanied by an instruction (which stratification to refine). For regulated
deployment, under GxP or SR 11-7~\cite{sr117}, an abstention that comes with a cycle
witness is a materially different artefact from an abstention that comes with a
confidence score.

\paragraph{Why cross-industry.} The credit instantiation is not decoration. It
demonstrates that obstruction can be a property of the \emph{context space}, not of the
data: because the macroeconomic regime axis is a circle, any model transported around
a full cycle is exposed to holonomy by construction. We suspect this is why
through-the-cycle credit models are recurrently surprised at regime turns, and we offer
it as a testable hypothesis rather than a claim.

\paragraph{Limitations.} (1) There is no real data in this paper. The simulations are controlled
validations of a theory, and while \S\ref{sec:mech} removes the circularity of the
original design by letting holonomy emerge from effect modification rather than
injecting it, an emergent mechanism in a model we wrote is still a model we wrote. The
paper should be read as theory with simulation support, and \S\ref{sec:prospective}
states the real-data study that would change that. (2) Correlated errors across overlaps are not handled: precision whitening
(\S\ref{sec:mech}) addresses unequal variances but not covariance, which would require
the full inverse-covariance metric and an estimate of it. (3) Evidence gaps and
obstruction interact, but not symmetrically---gaps \emph{hide} obstruction, so our
estimates are conservative; distinguishing ``no bridge'' from ``incoherent bridge''
still needs the sheaf-theoretic treatment with non-constant stalks. (4) We treat contexts
as given; eliciting the right factorisation is the hard applied problem and has no
theory here. (5) Real-valued claims only; categorical and ordinal claims require
cohomology with other coefficients, where the Hodge decomposition is unavailable.

\paragraph{Real-data validation.}\label{sec:prospective} The fastest credible real-data test does not require any
new collection. Published network meta-analyses report contrast estimates and their
standard errors; reconstructing the nerve from a published network, applying the
precision-whitened $F$-test and comparing against the node-splitting results the
original authors reported would place this framework directly against the established
method on its own ground, using only numbers already in print. A second test uses
existing open assets. \textsc{Medea}'s benchmarks, tools and the released yeast E-MAP
screen~\cite{medea} permit a direct test: construct the nerve over the 29 cell type
$\times$ 5 disease contexts, populate $\omega$ from PINNACLE and TranscriptFormer
contrasts, and test whether harmonic energy predicts (a) disagreement among
\textsc{MultiRoundDiscussion} panellists and (b) which of the agent's confident answers
are wrong. Prediction: the $79.1\%$ abstention rate of the literature-only ablation and
the $1.8\%$ rate of the LLM-only ablation bracket a harmonic-gated rate that is both
lower than the former and better targeted than the latter. If harmonic energy fails to
predict panel disagreement on real agent traces, the theory is wrong in the way that
matters.

\end{document}